\documentclass[11pt]{article}

\usepackage[T1]{fontenc}
\usepackage{lmodern}
\usepackage[x11names]{xcolor}
\usepackage{amsmath}
\usepackage{amssymb}
\usepackage{amsthm}
\usepackage{bm}
\usepackage[font={small}]{caption}
\usepackage[margin=1in]{geometry}
\usepackage{graphicx}
\usepackage[colorlinks=false, pdfborder={0 0 0}]{hyperref}
\usepackage{natbib}
\usepackage{tikz}

\newcommand*{\Reals}{\mathbb{R}}

\renewcommand*{\O}{\mathcal{O}}

\newcommand{\f}{\mathbf{f}}
\renewcommand{\v}{\mathbf{v}}
\newcommand{\w}{\mathbf{w}}
\newcommand{\x}{\mathbf{x}}
\newcommand{\y}{\mathbf{y}}

\newcommand{\stack}[2]{\genfrac{}{}{0pt}{}{#1}{#2}}

\newtheorem{theorem}{Theorem}
\newtheorem{lemma}{Lemma}
\newtheorem{corollary}{Corollary}

\theoremstyle{definition}
\newtheorem*{definition}{Definition}

\theoremstyle{remark}
\newtheorem*{remark}{Remark}

\bibpunct{(}{)}{;}{a}{}{,}

\usetikzlibrary{arrows, positioning}
\tikzstyle{activation} = [draw, rectangle, fill=green!20, minimum height=6mm, minimum width=9mm]
\tikzstyle{combo} = [draw, rectangle, fill=blue!20, minimum height=1cm, minimum width=1.2cm, font=\footnotesize]
\tikzstyle{knot} = [draw, fill=blue, blue, circle, minimum width=1mm, inner sep=0pt]
\tikzstyle{neuron} = [draw, rectangle, fill=red!20, minimum height=6mm, minimum width=6mm]
\tikzstyle{scalar} = [minimum height=7mm, minimum width=7mm]
\tikzstyle{subfig} = [font=\small, inner sep=0pt, anchor=north west]
\tikzstyle{xtick} = [pos=0, below, font=\scriptsize, black]

\title{The upper bound on knots in neural networks}
\author{%
    Kevin K. Chen\thanks{%
        Institute for Defense Analyses, Center for Communications Research - La Jolla%
    } \thanks{%
        Email for correspondence: \href{mailto:kkchen@ccrwest.org}{kkchen@ccrwest.org}%
    }
}

\date{November 2016}

\begin{document}

\maketitle

\begin{abstract}
    Neural networks with rectified linear unit activations are essentially multivariate linear splines.
    As such, one of many ways to measure the ``complexity'' or ``expressivity'' of a neural network is to count the number of knots in the spline model.
    We study the number of knots in fully-connected feedforward neural networks with rectified linear unit activation functions.
    We intentionally keep the neural networks very simple, so as to make theoretical analyses more approachable.
    An induction on the number of layers $l$ reveals a tight upper bound on the number of knots in $\mathbb{R} \to \mathbb{R}^p$ deep neural networks.
    With $n_i \gg 1$ neurons in layer $i = 1, \dots, l$, the upper bound is approximately $n_1 \dots n_l$.
    We then show that the exact upper bound is tight, and we demonstrate the upper bound with an example.
    The purpose of these analyses is to pave a path for understanding the behavior of general $\mathbb{R}^q \to \mathbb{R}^p$ neural networks.
\end{abstract}

\section{Introduction}

In recent years, neural networks---and deep neural networks in particular---have succeeded exceedingly well in such a great plethora of data-driven problems, so as to herald an entire paradigm shift in the way data science is approached.
Many everyday computerized tasks---such as image and optical character recognition, the personalization of Internet search results and advertisements, and even playing games such as chess, backgammon, and Go---have been deeply impacted and vastly improved by the application of neural networks.
The applications of neural networks, however, have advanced significantly more rapidly than the theoretical understanding of their successes.
Elements of neural network structures---such as the division of vector spaces into convex polytopes, and the application of nonlinear activation functions---afford neural networks a great flexibility to model many classes of functions with spectacular accuracy.
The flexibility is embodied in universal approximation theorems \citep{CybenkoMCSS89,HornikNN89,HornikNN91,SonodaACHA}, which essentially state that neural networks can model any continuous function arbitrarily well.
The complexity of neural networks, however, have also made their analytical understanding somewhat elusive.

The general thrust of this paper, as well as two companion papers \citep{WalkerSCAMP16, ChenNIPS16}, is to explore some unsolved elements of neural network theory, and to do so in a way that is independent of specific problems.
In the broadest sense, we seek to understand what models neural networks are capable of producing.
There exist many variations of neural networks, such as convolutional neural networks, recurrent neural networks, and long short-term memory models, each having their own arenas of success.
For simplicity, we choose to focus on the simplest case of feedforward, fully-connected neural networks with rectified linear unit activations.
This model is defined more precisely in \autoref{sec:neural-nets}.

More specifically, as we will see, neural networks with rectified linear unit activations are linear splines; i.e., they are continuous, piecewise linear functions with a finite number of pieces.
Therefore, one of many ways to measure of the ``complexity'' or ``expressivity'' of a neural network is to count the number of knots, i.e., discontinuities in the first derivative of the output quantities with respect to input quantities.
Similarly, one could count the number of piecewise linear regions given by the neural network.
Previous works \citep[e.g.,][]{MontufarNIPS14,Pascanu14,Raghu16} have observed or shown that number of piecewise linear pieces grows exponentially with the number of layers in the neural network, therefore justifying the use of deep networks over shallow networks.

In this paper, we continue the exploration of how the size of a neural network, given by the \emph{width} or the number of neurons in a layer, and the \emph{depth} or the number of layers, is related to the number of knots in the neural network.
Whereas previous works have generally focused on asymptotic or otherwise approximate upper bounds, we derive an exact tight upper bound.
The chief utility of such a bound is that it allows an \emph{a priori} determination of whether a neural network size is sufficient for a given task or governing equation.
For instance, we could imagine that a neural network designer at least roughly knows the complexity of the input--output behavior of a function to be modeled.
In this case, certain neural network widths and depths could be ruled out, on the grounds that no neural networks of those sizes could produce enough knots to model the function of interest.

In this paper, we attempt to circumvent some of the complexities of neural network behavior by making simplifications that may seem strong at times.
For instance, the results we report apply specifically to $\Reals \to \Reals^p$ functions.
Although neural networks are almost never used to study single-input functions, the simplicity does admit certain analyses that would otherwise be very difficult for general $\Reals^q \to \Reals^p$ functions.
Indeed, a key objective following this paper is to extend the results to multidimensional inputs.
This extension is tantamount to analyzing convex polytopes in $\Reals^q$ instead of linear segments in $\Reals$ in the input space.

The main results of the paper are given by the following theorems.

\begin{theorem}
    \label{thm:bound}
    In an $l$-layer $\Reals \to \Reals^p$ neural network with $n_i$ rectified linear unit neurons in layer $i = 1, \dots, l$, the number of knots $m_l$ in the neural network model satisfies
    \begin{equation}
        \label{eq:bound}
        m_l \le \sum_{i=1}^l n_i \prod_{j=i+1}^l (n_j + 1).
    \end{equation}
\end{theorem}

\begin{theorem}
    \label{thm:tight}
    If $n_i \ge 3$ for $i = 1, \dots, l - 1$ and $n_l \ge 2$, then the upper bound~\eqref{eq:bound} is tight.
\end{theorem}

This paper is organized as follows.
\autoref{sec:neural-nets} briefly reviews the neural network architecture that we employ in this paper.
Constructive proofs of Theorems~\ref{thm:bound} and~\ref{thm:tight} are presented respectively in Sections~\ref{sec:upper-bound} and~\ref{sec:tight}.
An example of a deep neural network meeting the upper bound on the number of knots is then constructed in \autoref{sec:example}.
Finally, we summarize our work and comment on future directions in \autoref{sec:conclusion}.

\section{Brief overview of neural networks}
\label{sec:neural-nets}

In \autoref{sec:description}, we first review the basic definitions and descriptions of neural networks.
Next, we describe two ideas which are relevant for the analytical development of the paper.
\autoref{sec:splines-knots-roots} describes the rectified linear unit neural network as a linear spline with associated knots and roots, so as to allow knot counting.
Afterwards, \autoref{sec:forward-relu} derives a transformation of the neural network into an equivalent model with only forward-facing rectified linear units.
Such a transformation is useful in constructing particular neural networks (e.g., for \autoref{thm:tight} and its associated lemmas).

\subsection{Description of neural networks}
\label{sec:description}

Neural networks are most commonly employed in the context of supervised machine learning, where the primary objective is to construct a function that best models a data set.
In this paper, however, we will be more concerned with the functional behavior of neural network models than with the training of such models.
As such, we will not address common topics such as model risk, loss, and optimization.
A review of machine learning techniques and their statistical analyses can be found in \citet{KnoxML}.

We begin by defining neural networks of a single or multiple hidden layers.
It is noteworthy that many variations on neural networks exist.
The definitions below correspond to the dense, fully-connected, feedforward structure we will employ, but may differ from architectures used in other studies or applications.

\begin{definition}
    For some \emph{bias} $b \in \Reals$, \emph{weight} $\w \in \Reals^n$, nonlinear \emph{activation function} $\sigma : \Reals \to \Reals$, and \emph{input} $\v \in \Reals^n$, a \emph{neuron} is the function $\sigma(\w \cdot \v + b)$.
\end{definition}

\begin{definition}
    Let $q$ and $p$ respectively denote the input and output dimension.
    For $k = 1, \dots, n$, with $n$ the number of neurons, select \emph{input biases} $b_{1k} \in \Reals$ and \emph{input weights} $\w_{1k} \in \Reals^q$.
    Also, for $k = 1, \dots, p$, select \emph{output biases} $b_{2k} \in \Reals$ and \emph{output weights} $\w_{2k} \in \Reals^n$.
    Using the shorthand notation $\v := [v_1 \; \cdots \; v_n] \in \Reals^n$ and $\y := [y_1 \; \cdots \; y_p] \in \Reals^p$, a \emph{single-hidden-layer} neural network is the model $\hat{\f} : \Reals^q \to \Reals^p$, $\x \mapsto \y$ given by
    \begin{subequations}
        \label{eq:shallow}
        \begin{align}
            \label{eq:shallow-input}
            v_k &:= \sigma(\w_{1k} \cdot \x + b_{1k}), \quad
            k = 1, \dots, n, \\
            \label{eq:shallow-output}
            y_k &:= \w_{2k} \cdot \v + b_{2k}, \quad
            k = 1, \dots, p.
        \end{align}
    \end{subequations}
\end{definition}

This architecture is shown in \autoref{fig:1layer}.
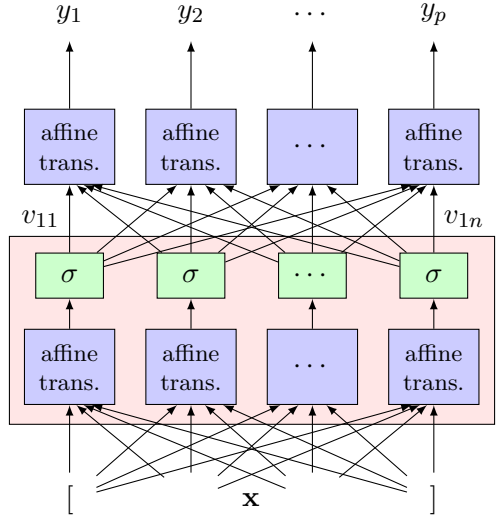
\begin{figure}[t]
    \centering

    \begin{tikzpicture}[>=latex, node distance=9mm]
        % Inputs.
        \node (x1) [scalar] {[}; % {$x_1$};
        \node (x2) [scalar, right=of x1] {}; % {$x_2$};
        \node (xi) [scalar, right=of x2] {}; % {$\cdots$};
        \node (xq) [scalar, right=of xi] {]}; % {$x_q$};
        \node [right=1.9mm of x2] {$\x$};

        % Hidden layer.
        \draw [fill=red!10] (-0.8, 1) rectangle (5.65, 3.5);

        % Hidden layer combinations.
        \node (combo x1) [combo, above=of x1, align=center] {affine \\ trans.};
        \node (combo x2) [combo, above=of x2, align=center] {affine \\ trans.};
        \node (combo xi) [combo, above=of xi, align=center] {\normalsize $\cdots$};
        \node (combo xq) [combo, above=of xq, align=center] {affine \\ trans.};

        % Hidden layer activations.
        \node (sigma1) [activation, above=4mm of combo x1] {$\sigma$};
        \node (sigma2) [activation, above=4mm of combo x2] {$\sigma$};
        \node (sigmai) [activation, above=4mm of combo xi] {$\cdots$};
        \node (sigman) [activation, above=4mm of combo xq] {$\sigma$};

        % Output combinations.
        \node (combo y1) [combo, above=of sigma1, align=center] {affine \\ trans.};
        \node (combo y2) [combo, above=of sigma2, align=center] {affine \\ trans.};
        \node (combo yi) [combo, above=of sigmai, align=center] {\normalsize $\cdots$};
        \node (combo yp) [combo, above=of sigman, align=center] {affine \\ trans.};

        % Outputs.
        \node (y1) [scalar, above=of combo y1] {$y_1$};
        \node (y2) [scalar, above=of combo y2] {$y_2$};
        \node (yi) [scalar, above=of combo yi] {$\cdots$};
        \node (yp) [scalar, above=of combo yp] {$y_p$};

        % Inputs to combos.
        \draw [->] (x1) -- (combo x1.270);
        \draw [->] (x2) -- (combo x1.280);
        \draw [->] (xi) -- (combo x1.290);
        \draw [->] (xq) -- (combo x1.300);

        \draw [->] (x1) -- (combo xi.230);
        \draw [->] (x2) -- (combo xi.250);
        \draw [->] (xi) -- (combo xi.270);
        \draw [->] (xq) -- (combo xi.290);

        \draw [->] (x1) -- (combo xq.240);
        \draw [->] (x2) -- (combo xq.250);
        \draw [->] (xi) -- (combo xq.260);
        \draw [->] (xq) -- (combo xq.270);

        \draw [->] (x1) -- (combo x2.250);
        \draw [->] (x2) -- (combo x2.270);
        \draw [->] (xi) -- (combo x2.290);
        \draw [->] (xq) -- (combo x2.310);

        % Hidden layer combos to activations.
        \draw [->] (combo x1) -- (sigma1);
        \draw [->] (combo x2) -- (sigma2);
        \draw [->] (combo xi) -- (sigmai);
        \draw [->] (combo xq) -- (sigman);

        % Hidden layer activations to output combos.
        \draw [->] (sigma1) -- (combo y2.250);
        \draw [->] (sigma2) -- (combo y2.270);
        \draw [->] (sigmai) -- (combo y2.290);
        \draw [->] (sigman) -- (combo y2.310);

        \draw [->] (sigma1) -- (combo yi.230);
        \draw [->] (sigma2) -- (combo yi.250);
        \draw [->] (sigmai) -- (combo yi.270);
        \draw [->] (sigman) -- (combo yi.290);

        \draw [->] (sigma1) -- (combo yp.240);
        \draw [->] (sigma2) -- (combo yp.250);
        \draw [->] (sigmai) -- (combo yp.260);
        \draw [->] (sigman) -- node [right] {$v_{1n}$} (combo yp.270);

        \draw [->] (sigma1) -- node [left] {$v_{11}$} (combo y1.270);
        \draw [->] (sigma2) -- (combo y1.280);
        \draw [->] (sigmai) -- (combo y1.290);
        \draw [->] (sigman) -- (combo y1.300);

        % Output combos to outputs.
        \draw [->] (combo y1) -- (y1);
        \draw [->] (combo y2) -- (y2);
        \draw [->] (combo yi) -- (yi);
        \draw [->] (combo yp) -- (yp);
    \end{tikzpicture}

    \caption{%
        The single-hidden-layer neural network, with the hidden layer shown in red.%
    }
    \label{fig:1layer}
\end{figure}
In summary, each neuron takes an affine transformation of the input and applies the activation function~\eqref{eq:shallow-input}.
Then, each output takes an affine transformation of all the neural outputs~\eqref{eq:shallow-output}.
The flexibility of this architecture is apparent from the $(q + 1) n + (n + 1) p$ scalars that comprise the biases and weights.
In particular, the well-known universal approximation theorem loosely states that if the activation function $\sigma$ is continuous, non-constant, and bounded, then the single-hidden-layer neural network can approximate any continuous function arbitrarily well with a finite number $n$ of neurons \citep{CybenkoMCSS89,HornikNN89,HornikNN91}.
A recent result \citep{SonodaACHA} extends the universal approximation result to the commonly employed \emph{rectified linear unit}
\begin{equation}
    \label{eq:relu}
    \sigma(x) := \max(0, x) = \frac{x + |x|}{2}.
\end{equation}

Although the universal approximation theorem implies that the single-hidden-layer neural network is sufficiently flexible for modeling continuous functions, it is common to employ \emph{deep neural networks}, where the outputs of neurons are fed into further hidden layers of neurons.
Such architectures are behind many of the notable successes in machine learning applications.
The deep neural network with $l$ layers proceeds as follows.

\begin{definition}
    Let $q$ and $p$ respectively denote the input and output dimension.
    Set $n_i$ as the number of neurons for each layer $i = 1, \dots, l$.
    For $k = 1, \dots, n_1$, select input weight vectors $\w_{1k} \in \Reals^q$ and input biases $b_{1k} \in \Reals$.
    Also, for $i = 2, \dots, l$ and for each $k = 1, \dots, n_i$, also select weight vectors $\w_{ik} \in \Reals^{n_{i-1}}$ and biases $b_{ik} \in \Reals$.
    Finally, for $k = 1, \dots, p$, select output weight vectors $\w_{l+1, k} \in \Reals^{n_l}$ and output biases $b_{l+1, k} \in \Reals$.
    Using the shorthand notation $\v_i := [v_{i1} \; \cdots \; v_{i n_i}] \in \Reals^{n_i}$ and $\y := [y_1 \; \cdots \; y_p]$,
    a \emph{deep neural network} is the model $\hat{\f} : \Reals^q \to \Reals^p$, $\x \mapsto \y$ given by
    \begin{subequations}
        \label{eq:deep}
        \begin{align}
            \label{eq:deep-input}
            v_{1k} &:= \sigma(\w_{1k} \cdot \x + b_{1k}), \quad
            k = 1, \dots, n_1 \\
            v_{ik} &:= \sigma(\w_{ik} \cdot \v_{i-1} + b_{ik}), \quad
            i = 2, \dots, l, \quad
            k = 1, \dots, n_i \\
            \label{eq:deep-output}
            y_k &:= \w_{l+1, k} \cdot \v_l + b_{l+1, k}, \quad
            k = 1, \dots, p.
        \end{align}
    \end{subequations}
\end{definition}

The deep neural network architecture is shown in \autoref{fig:deep}.
\begin{figure}[t]
    \centering

    \begin{tikzpicture}[>=latex, node distance=6mm]
        % Layer 1.
        \node (n11) [neuron] {};
        \node (n12) [neuron, right=of n11] {};
        \node (n13) [neuron, right=of n12] {};
        \node (n14) [neuron, right=of n13] {};
        \node (n15) [neuron, right=of n14] {};

        % Inputs.
        \node (x1) [scalar, below=of n11, xshift=6mm] {[};
        \node (x2) [scalar, below=of n12, xshift=6mm] {};
        \node (xi) [scalar, below=of n13, xshift=6mm] {};
        \node (xq) [scalar, below=of n14, xshift=6mm] {]};
        \node [below=of n13, yshift=-1mm] {$\x$};

        \node (v1) [above=1.1mm of n15, xshift=1mm] {$\v_1$};

        \draw [->] (x1.90) -- (n11.270);
        \draw [->] (x1.90) -- (n12.270);
        \draw [->] (x1.90) -- (n13.270);
        \draw [->] (x1.90) -- (n14.270);
        \draw [->] (x1.90) -- (n15.270);

        \draw [->] (x2.90) -- (n11.270);
        \draw [->] (x2.90) -- (n12.270);
        \draw [->] (x2.90) -- (n13.270);
        \draw [->] (x2.90) -- (n14.270);
        \draw [->] (x2.90) -- (n15.270);

        \draw [->] (xi.90) -- (n11.270);
        \draw [->] (xi.90) -- (n12.270);
        \draw [->] (xi.90) -- (n13.270);
        \draw [->] (xi.90) -- (n14.270);
        \draw [->] (xi.90) -- (n15.270);

        \draw [->] (xq.90) -- (n11.270);
        \draw [->] (xq.90) -- (n12.270);
        \draw [->] (xq.90) -- (n13.270);
        \draw [->] (xq.90) -- (n14.270);
        \draw [->] (xq.90) -- (n15.270);

        % Layer 2.
        \node (n21) [neuron, above=of n11, xshift=6mm] {};
        \node (n22) [neuron, right=of n21] {};
        \node (n23) [neuron, right=of n22] {};
        \node (n24) [neuron, right=of n23] {};

        \node (v2) [above=1.1mm of n24, xshift=1mm] {$\v_2$};

        \draw [->] (n11.90) -- (n21.270);
        \draw [->] (n11.90) -- (n22.270);
        \draw [->] (n11.90) -- (n23.270);
        \draw [->] (n11.90) -- (n24.270);

        \draw [->] (n12.90) -- (n21.270);
        \draw [->] (n12.90) -- (n22.270);
        \draw [->] (n12.90) -- (n23.270);
        \draw [->] (n12.90) -- (n24.270);

        \draw [->] (n13.90) -- (n21.270);
        \draw [->] (n13.90) -- (n22.270);
        \draw [->] (n13.90) -- (n23.270);
        \draw [->] (n13.90) -- (n24.270);

        \draw [->] (n14.90) -- (n21.270);
        \draw [->] (n14.90) -- (n22.270);
        \draw [->] (n14.90) -- (n23.270);
        \draw [->] (n14.90) -- (n24.270);

        \draw [->] (n15.90) -- (n21.270);
        \draw [->] (n15.90) -- (n22.270);
        \draw [->] (n15.90) -- (n23.270);
        \draw [->] (n15.90) -- (n24.270);

        % Etc.
        \node (etc) [above=8mm of n22, xshift=6mm] {$\vdots$};

        \node (n31) [minimum width=6mm, above=of n21, xshift=6mm] {};
        \node (n32) [minimum width=6mm, right=of n31] {};
        \node (n33) [minimum width=6mm, right=of n32] {};

        \draw [->] (n21.90) -- (n31.270);
        \draw [->] (n21.90) -- (n32.270);
        \draw [->] (n21.90) -- (n33.270);

        \draw [->] (n22.90) -- (n31.270);
        \draw [->] (n22.90) -- (n32.270);
        \draw [->] (n22.90) -- (n33.270);

        \draw [->] (n23.90) -- (n31.270);
        \draw [->] (n23.90) -- (n32.270);
        \draw [->] (n23.90) -- (n33.270);

        \draw [->] (n24.90) -- (n31.270);
        \draw [->] (n24.90) -- (n32.270);
        \draw [->] (n24.90) -- (n33.270);

        % Layer l.
        \node (n41) [above=4mm of n31] {};
        \node (n42) [above=4mm of n32] {};
        \node (n43) [above=4mm of n33] {};

        \node (nl1) [neuron, above=of etc, xshift=-6mm] {};
        \node (nl2) [neuron, right=of nl1] {};

        \node (vl) [above=-2mm of nl2, xshift=1.1cm] {$\v_l$};

        \draw [->] (n41.90) -- (nl1.270);
        \draw [->] (n41.90) -- (nl2.270);

        \draw [->] (n42.90) -- (nl1.270);
        \draw [->] (n42.90) -- (nl2.270);

        \draw [->] (n43.90) -- (nl1.270);
        \draw [->] (n43.90) -- (nl2.270);

        % Outputs.
        \node (y2) [scalar, above=of nl1] {$y_2$};
        \node (yi) [scalar, above=of nl2] {$\cdots$};
        \node (y1) [scalar, left=of y2] {$y_1$};
        \node (yp) [scalar, right=of yi] {$y_p$};

        \draw [->] (nl1.90) -- (y1.270);
        \draw [->] (nl1.90) -- (y2.270);
        \draw [->] (nl1.90) -- (yi.270);
        \draw [->] (nl1.90) -- (yp.270);

        \draw [->] (nl2.90) -- (y1.270);
        \draw [->] (nl2.90) -- (y2.270);
        \draw [->] (nl2.90) -- (yi.270);
        \draw [->] (nl2.90) -- (yp.270);
    \end{tikzpicture}

    \caption{The deep neural network, with each neuron shown in red.}
    \label{fig:deep}
\end{figure}
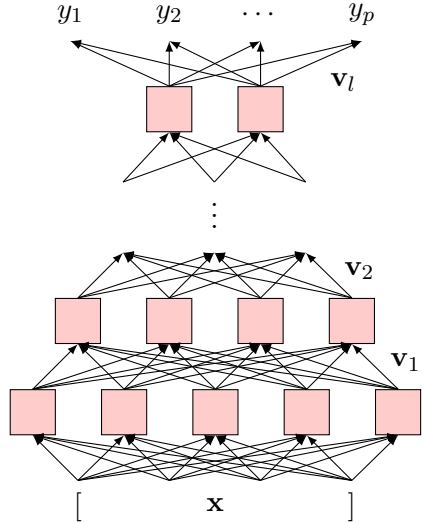
Typically, $n_1 > \dots > n_l$; it has been empirically shown that training risk is better reduced by optimizing layers closer to the input than layers closer to the output \citep{Raghu16}.

\subsection{Splines, knots, and roots}
\label{sec:splines-knots-roots}

In this study, we will use the rectified linear unit activation function \eqref{eq:relu} in all neurons.
The rectified linear unit is a common choice because it creates flexible models and is fast to compute.
Other common choices, such as the sigmoid function $1 / (1 + e^{-x})$, are more computationally intensive.
They also typically have smaller regions in the domain where the first derivative is far from zero, which can pose additional challenges when training neural networks on data.

With the rectified linear unit activation, the neural network is essentially a linear spline.
To understand this property, first consider the simplified case of a single scalar input, i.e., where the neural network is some $\hat{\f} : \Reals \to \Reals^p, x \mapsto \y$.
The outputs of the first hidden layer~(\ref{eq:shallow-input},~\ref{eq:deep-input}) are $v_{1k}(x) = \sigma(w_{1k} x + b_{1k})$ for $k = 1, \dots, n_1$.
Since $\sigma(x)$ is continuous and has a discontinuity in $d\sigma / dx$ at $x = 0$, $v_{1k}$ is clearly also continuous and has a discontinuity in $dv_{1k} / dx$ at $x = -b_{1k} / w_{1k}$.
Thus, the functions $v_{1k}(x)$ are linear splines.
The next layer, whether it is a second hidden layer or the output layer, then computes an affine transformation of the functions $v_{1k}(x)$.
Such an affine transformation is continuous; hence, it is still a linear spline.
This reasoning can be carried out through each hidden layer to the output.

In every application of a rectified linear unit beyond the first layer, knots can be retained, destroyed, or created.
An example of this process is shown in \autoref{fig:spline_relu}
\begin{figure}[t]
    \centering
    \includegraphics{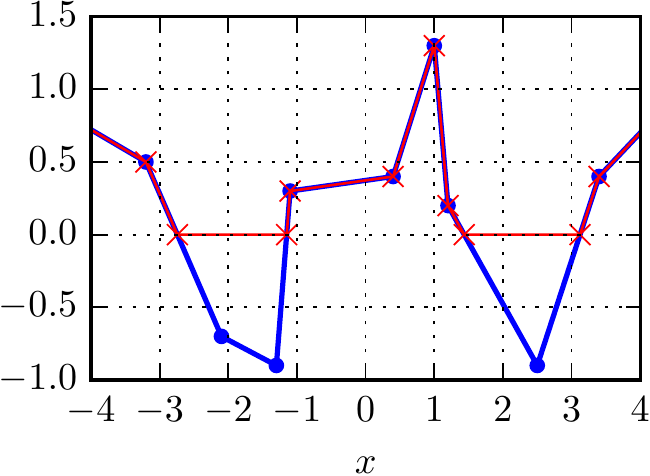}

    \caption{%
        Blue: an example of the affine transformation $\w_{ik} \cdot \v_{i-1}(x) + b_{ik}$ in neuron $k$ of layer $i$.
        The knots (filled dots) originate from the various scalar elements of $\v_{i-1}(x)$.
        Red: the neural output $\sigma(\w_{ik} \cdot \v_{i-1}(x) + b_{ik})$, with knots shown as $\times$.
        Knots of the blue spline above zero are retained, knots below zero are discarded, and roots of the blue spline appear as new knots.%
    }
    \label{fig:spline_relu}
\end{figure}
for some neuron $k$ in some layer $i$.
If the previous layer output $\v_{i-1}(x)$ contains a particular knot at some $x_j$ such that $\w_{ik} \cdot \v_{i-1}(x_j) + b_{ik} > 0$, then the application of the rectified linear unit does not alter this knot, and the knot is retained by this neuron.
On the other hand, if $\w_{ik} \cdot \v_{i-1}(x_j) + b_{ik} < 0$, then both the knot and the immediate neighborhood of $x_j$ are rectified to zero, and the knot at $x_j$ is destroyed.
Finally, wherever $\w_{ik} \cdot \v_{i-1}(x) + b_{ik}$ crosses zero, there exists a region on one side of the root that is rectified to zero.
The rectification introduces a new knot at the root, as shown in \autoref{fig:spline_relu}.

In all three cases, the neural output $\sigma(\w_{ik} \cdot \v_{i-1}(x) + b_{ik})$ again remains a continuous function with discrete discontinuities in its first derivative.
Hence, even deep neural networks with rectified linear unit activations are linear splines.
The mechanisms for retaining, destroying, and creating knots will be relevant when deriving the upper bound on the number of knots in \autoref{sec:upper-bound}.
The description of knots becomes more sophisticated in the typical
scenario where the input space is $\Reals^q$ with $q > 1$.  In this
case, each neuron in the first hidden layer divides the input space
into two regions split by the hyperplane $\w_{1k} \cdot \x + b_{1k}=
0$.  With the rectified linear unit acting on $\w_{1k} \cdot \x +
b_{1k}$, each neuron outputs zero on one side of the hyperplane, and a
half-plane with normal vector $[\x \; v_{1k}] = [-\w_{1k} \; 1]$ on
the other side.  Just as further hidden layers retain, destroy, and
create new knots for $q = 1$, further hidden layers retain, destroy,
and create new hyperplanes or pieces thereof for\break $q > 1$.  The
resulting neural network is a piecewise linear $\Reals^q \to \Reals^p$
model on a finite number of convex polytopes; see
\autoref{fig:2d-model}
\begin{figure}[t]
    \centering
    \includegraphics[width=3in]{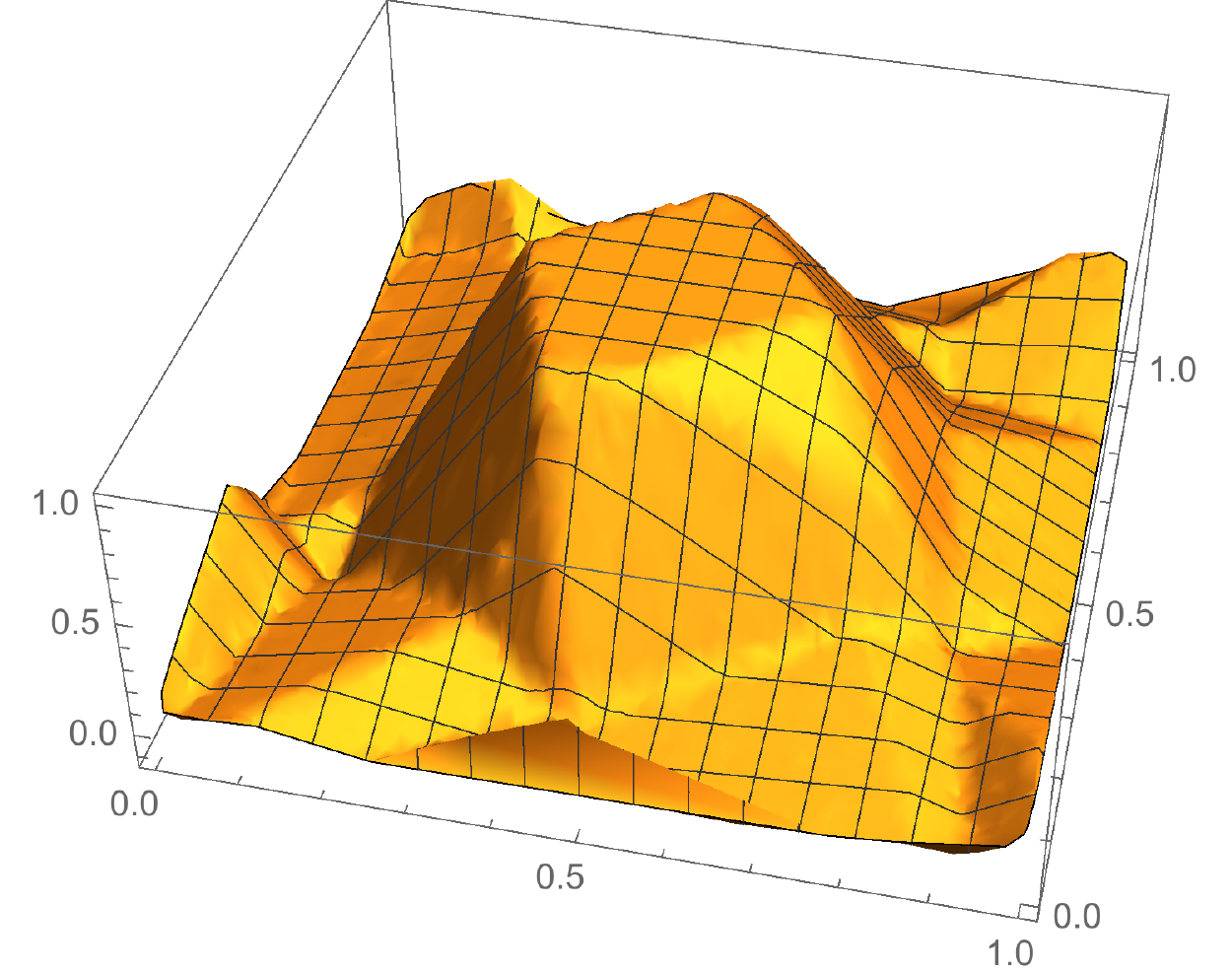}
    \caption{A $\Reals^2 \to \Reals$ neural network model.}
    \label{fig:2d-model}
\end{figure}
for an example.
It is still possible to analyze such neural networks in a one-dimensional sense if we were to consider one-dimensional trajectories through the input space $\Reals^q$ \citep{Raghu16}, but the full model is notably more complex in general.
Many analytical results on multidimensional input spaces rely on upper bounds and asymptotics based on polytope counting \citep{MontufarNIPS14,Pascanu14,Raghu16}.

\subsection{Equivalent form with forward-facing rectified linear units}
\label{sec:forward-relu}

For the simple case of $\Reals \to \Reals^p$ neural networks with one hidden layer, the neurons in the first hidden layer~(\ref{eq:shallow-input},~\ref{eq:deep-input}) output $v_{1k} = \sigma(w_{1k} x + b_{1k})$, which is essentially a rectified linear unit $\sigma(x)$ that is horizontally stretched and translated, and possibly reflected across the $v_{1k}$-axis.
Therefore, the sloped ray in the activated region can extend into quadrants I or II in the $x$--$v_{1k}$ plane.
For the purpose of constructing or analyzing $\Reals \to \Reals^p$ neural networks, it is convenient to have all rectified linear units extend in the positive $x$ direction (i.e., into quadrant I), which we call ``forward-facing.''
Such a feature allows us to consider the action of each rectified linear unit by starting at $x = -\infty$ and increasing $x$.
Thus, no rectified linear units are activated at $x = -\infty$, and the units are successively activated with increasing $x$; no units are deactivated.

The transformation that expresses the scalar-input, single-hidden-layer neural network with forward-facing rectified linear units is as follows.

\begin{lemma}
    \label{lem:forward-relu}

    Consider the single-hidden-layer $\Reals \to \Reals^p$ rectified linear unit neural network with input weights $w_{1j} \in \Reals$, input biases $b_{1j} \in \Reals$, output weights $w_{2kj} \in \Reals$, and output biases $b_{2k} \in \Reals$ for $j = 1, \dots, n$ and $k = 1, \dots, p$.
    The neural network model
    \begin{equation}
        \label{eq:shallow-nn}
        y_k(x) = \sum_{j=1}^n w_{2kj} \sigma(w_{1j} x + b_{1j}) + b_{2k}, \quad
        k = 1, \dots, p
    \end{equation}
    (cf.~\eqref{eq:shallow} with $\w_{2k} = [w_{2k1} \; \dots \; w_{2kn}]$) is equivalently
    \begin{equation}
        \label{eq:forward-relu}
        y_k(x) = \sum_{j=1}^n s_{kj} \sigma(x - x_j) + c_{1k} x + c_{0k}, \quad
        k = 1, \dots, p,
    \end{equation}
    where
    \begin{subequations}
        \label{eq:forward-relu-parameters}
        \begin{align}
            c_{1k} &:= \sum_{\stack{1 \le j \le n}{w_{1j} < 0}} w_{2kj} w_{1j}, &
            c_{0k} &:= \sum_{\stack{1 \le j \le n}{w_{1j} < 0}} w_{2kj} b_{1j} + b_{2k}, & \\
            s_{kj} &:= w_{2kj} |w_{1j}|, &
            x_j &:= -\frac{b_{1j}}{w_{1j}}, &
            j &= 1, \dots, n
        \end{align}
    \end{subequations}
    for $k = 1, \dots, p$.
    All rectified linear units in~\eqref{eq:forward-relu} face forward.
\end{lemma}

\begin{proof}
    We first split the sum in~\eqref{eq:shallow-nn} according to the sign of $w_{1j}$, so that
    \begin{equation}
        y_k(x) =
        \sum_{\stack{1 \le j \le n}{w_{1j} < 0}} w_{2kj} \sigma(w_{1j} x + b_{1j}) +
        \sum_{\stack{1 \le j \le n}{w_{1j} \ge 0}} w_{2kj} \sigma(w_{1j} x + b_{1j}) +
        b_{2k}.
    \end{equation}
    Next, we observe from~\eqref{eq:relu} that
    \begin{equation}
        \label{eq:relu-negative}
        \sigma(x) = \sigma(-x) + x;
    \end{equation}
    using this property on the first sum, we obtain
    \begin{subequations}
        \begin{align}
            \begin{split}
                y_k(x) &=
                \sum_{\stack{1 \le j \le n}{w_{1j} < 0}} w_{2kj} \sigma(-w_{1j} x - b_{1j})
                + \sum_{\stack{1 \le j \le n}{w_{1j} < 0}} w_{2kj} (w_{1j} x + b_{1j}) \\
                & \hspace{1.2em}
                + \sum_{\stack{1 \le j \le n}{w_{1j} \ge 0}} w_{2kj} \sigma(w_{1j} x + b_{1j})
                + b_{2k}
            \end{split} \\
            &= \sum_{\stack{1 \le j \le n}{w_{1j} < 0}} w_{2kj} \sigma(-w_{1j} x - b_{1j})
            + \sum_{\stack{1 \le j \le n}{w_{1j} \ge 0}} w_{2kj} \sigma(w_{1j} x + b_{1j})
            + c_{1k} x
            + c_{0k}.
        \end{align}
    \end{subequations}

    To combine the two sums, we further observe that if $w \ge 0$, then $\sigma(w x) = w \sigma(x)$.
    Thus, we can pull $-w_{1j}$ out of the rectified linear unit in the first sum and $w_{1j}$ out of the same in the second sum, and obtain
    \begin{subequations}
        \begin{align}
            \begin{split}
                y_k(x) &=
                \sum_{\stack{1 \le j \le n}{w_{1j} < 0}} -w_{2kj} w_{1j} \sigma\left(x + \frac{b_{1j}}{w_{1j}}\right)
                + \sum_{\stack{1 \le j \le n}{w_{1j} \ge 0}} w_{2kj} w_{1j} \sigma\left(x + \frac{b_{1j}}{w_{1j}}\right) \\
                & \hspace{1.2em}
                + c_{1k} x
                + c_{0k}
            \end{split} \\
            &= \sum_{j=1}^n w_{2kj} |w_{1j}| \sigma\left(x + \frac{b_{1j}}{w_{1j}}\right) + c_{1k} x + c_{0k},
        \end{align}
    \end{subequations}
    which is equal to~\eqref{eq:forward-relu}.
    All rectified linear units face forward because the coefficient on $x$ is simply unity.
\end{proof}

Besides that all the rectified linear units in~\eqref{eq:forward-relu} face forward, the utility of that expression is that the entire neural network is expressed in terms of four sets of parameters~\eqref{eq:forward-relu-parameters}, each with a natural interpretation.
The parameter $x_j$ is the location of the knot created by neuron $j$.
For convenience, we will assume hereafter that all parameters in $j$ (i.e., $w_{1j}$, $b_{1j}$, $w_{2kj}$, $s_{kj}$, and $x_j$) are sorted by ascending $x_j$.
Next, in the contribution from the forward-facing rectified linear unit in neuron $j$ to the scalar output $k$, $s_{kj}$ is the slope of the activated region.
Finally, $c_{1k}$ and $c_{0k}$ describe the line that is added to the sum of rectified linear units, so as to complete the equivalence between~\eqref{eq:shallow-nn} and~\eqref{eq:forward-relu}.

\section{Upper bound on number of knots}
\label{sec:upper-bound}

Some recent articles have derived asymptotic or otherwise approximate upper bounds for the number of linear regions in neural networks with multidimensional inputs and outputs.
For instance, building on \citet{Pascanu14}, \citet{MontufarNIPS14} showed that for an $\Reals^q \to \Reals^p$ neural network with $n_i \ge q$ neurons in layer $i = 1, \dots, l$, the upper bound on the number of linear regions is at least
\begin{equation}
    \left(\prod_{i=1}^{l-1} \left\lfloor \frac{n_i}{q} \right\rfloor^q\right)
    \sum_{j=0}^q \binom{n_l}{j}.
\end{equation}
Later, \citet{Raghu16} gave asymptotic upper bounds for the number of linear regions in neural networks with multidimensional inputs and outputs.
The article shows that an $\Reals^q \to \Reals^p$ neural network with $n$ neurons in each of $l$ layers has a number of regions that grows at most like $\O(n^{q l})$ for rectified linear unit activations, and $\O((2 n)^{q l})$ for step activation functions.
Furthermore, the asymptotic upper bound is shown to be tight \citep{MontufarNIPS14,Pascanu14,Raghu16}.

In this section, we derive an exact as opposed to asymptotic or approximate upper bound, but restrict ourselves to the case of $\Reals \to \Reals^p$ neural networks.
The possibility of extending the result to $\Reals^q \to \Reals^p$ remains open.
We first discuss the mechanisms by which the maximal number of knots is retained and created in each hidden layer.
Next, we use induction to prove \autoref{thm:bound}, which states the upper bound.
Afterwards, we prove in \autoref{sec:tight} that the upper bound is tight (\autoref{thm:tight}).

We begin with a basic definition that we will use throughout this section.

\begin{definition}
    A knot or its location is \emph{unique} if the knot's input coordinate is different from that of all other knots in the neural network.
\end{definition}

To set the base case for the induction, we first consider the neural network with $l = 1$ layer and $n_1$ neurons in that layer.
Using the notation of Lemma~\ref{lem:forward-relu}, we make the simple observation that in a one-hidden-layer neural network, each neuron contributes exactly one knot to the model at $x_j = -b_{1j} / w_{1j}$.
If the input biases $b_{1j}$ and input weights $w_{1j}$ are selected such that the knot locations $x_j$ are unique, then the neural network has exactly $n_1$ knots.

To consider the inductive step, recall from \autoref{sec:splines-knots-roots} that every application of a rectified linear unit can preserve, destroy, or create new knots.
For the purposes of constructing an upper bound, we can make the stronger statement that with the proper choice of weights and biases, every knot can be preserved in every hidden layer.
Explicitly, the knots in the affine transformation $\w_{ik} \cdot \v_{i-1}(x) + b_{ik}$ of layer $i - 1$ outputs can be preserved in $\sigma(\w_{ik} \cdot \v_{i-1}(x) + b_{ik})$, the output of neuron $k$ in layer $i$.
The most naive way to do so is to set the biases $b_{ik}$ so high that $\w_{ik} \cdot \v_{i-1}(x_j) + b_{ik} > 0$ for all knots $x_j$; see \autoref{fig:sawtooth}(a).
\begin{figure}[t]
    \centering

    \begin{tikzpicture}[x=1.25cm, y=1.25cm, >=latex]
        % Neuron with high bias.
        \def\x{0}
        \def\y{2}
        \foreach \xshift in {0, 1, 2, 3} {%
            \draw [thick] (\x + \xshift, \y - 0.5) -- (\x + \xshift + 0.5, \y + 0.5) -- (\x + \xshift + 1, \y - 0.5);
            \draw [fill=red] (\x + \xshift + 0.5, \y + 0.5) circle (0.8mm);
        }
        \foreach \xshift in {1, 2, 3} {%
            \draw [fill=red] (\x + \xshift, \y - 0.5) circle (0.8mm);
        }

        \draw [dashed] (0, \y - 1) -- (4, \y - 1);
        \node [subfig] at (\x - 0.2, \y + 0.6) {(a)};

        % Two neurons with signs flipped.
        \def\x{5}
        \foreach \xshift in {0, 1, 2, 3} {%
            \draw [thick] (\x + \xshift, \y - 0.5) -- (\x + \xshift + 0.5, \y + 0.5) -- (\x + \xshift + 1, \y - 0.5);
            \draw [fill=red] (\x + \xshift + 0.5, \y + 0.5) circle (0.8mm);
        }
        \foreach \xshift in {1, 2, 3} {%
            \draw [fill=blue] (\x + \xshift, \y - 0.5) circle (0.8mm);
        }

        \def\dy{0.08}
        \foreach \sign in {-1, 1} {%
            \draw [dashed] (\x, \y + \sign * \dy) -- (\x + 5.5, \y + \sign * \dy);
        }
        \draw [->, ultra thick, red] (\x + 4.2, \y + \dy) -- node [right] {\small neuron 1} (\x + 4.2, \y + \dy + 0.5);
        \draw [->, ultra thick, blue] (\x + 4.2, \y - \dy) -- node [right] {\small neuron 2} (\x + 4.2, \y - \dy - 0.5);

        \node [subfig] at (\x - 0.2, \y + 0.6) {(b)};

        % Roots.
        \def\x{2.5}
        \def\y{0}
        \def\dy{0.25}

        \draw [dashed] (\x + 0, \y + \dy) -- node [right, pos=1] {\footnotesize $k = 1$} (\x + 4, \y + \dy);
        \draw [dashed] (\x + 0, \y) -- node [right, pos=1] {\footnotesize $k = 2$} (\x + 4, \y);
        \draw [dashed] (\x + 0, \y - \dy) -- node [right, pos=1] {\footnotesize $k = 3$} (\x + 4, \y - \dy);

        \foreach \xshift in {0, 1, 2, 3} {%
            \draw [thick] (\x + \xshift, \y - 0.5) -- (\x + \xshift + 0.5, \y + 0.5) -- (\x + \xshift + 1, \y - 0.5);

            \foreach \dx in {-0.1, 0, 0.1} {%
                \draw [fill=Green4] (\x + \xshift + 0.25 + \dx, \y + 2 * \dx) circle (0.8mm);
                \draw [fill=Green4] (\x + \xshift + 0.75 - \dx, \y + 2 * \dx) circle (0.8mm);
            }
        }

        \node [subfig] at (\x - 0.2, \y + 0.6) {(c)};
    \end{tikzpicture}

    \caption{%
        Schematics for preserving and creating knots in neuron $k$ of layer $i$.
        (a)~All knots $x_j$ in $\w_{ik} \cdot \v_{i-1}(x) + b_{ik}$ (red) can be preserved in $\sigma(\w_{ik} \cdot \v_{i-1}(x) + b_{ik})$ by setting $b_{ik}$ sufficiently high so that $\w_{ik} \cdot \v_{i-1}(x_j) + b_{ik}$ is greater than zero (dashed line) for all $k$.
        (b)~Alternatively, two neurons (red and blue) can assign similar weights and biases with opposite signs to preserve knots on both sides of zero.
        (c)~If $\w_{ik} \cdot \v_{i-1}(x) + b_{ik}$ is a sawtooth wave with $m_{i-1}$ knots, then each neuron $k$ in layer $i$ can uniquely create $m_{i-1} + 1$ new knots.
        An example is shown for $k = 1, 2, 3$.%
    }
    \label{fig:sawtooth}
\end{figure}
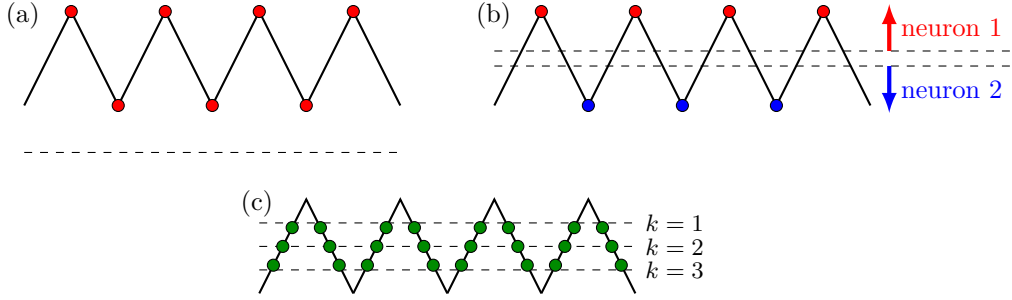
The disadvantage of this method is that the rectified linear unit does not create any new knots.
A better but still very simple alternative is to have two neurons in layer $i$ employ identical or similar weights $\w_{ik}$ and biases $b_{ik}$, but with flipped signs.
This way, as shown in \autoref{fig:sawtooth}(b), one neuron would preserve some subset of the knots of $\w_{ik} \cdot \v_{i-1}(x) + b_{ik}$, and the other neuron would preserve the complement.
With this design, each rectified linear unit is able to create the maximum possible number of knots as follows.

Since each affine transformation $\w_{ik} \cdot \v_{i-1}(x) + b_{ik}$ is a linear spline, each line segment between adjacent knots can have at most one root.
If $\w_{ik} \cdot \v_{i-1}(x) + b_{ik}$ has $m_{i-1}$ knots, then these connections can cumulatively have at most $m_{i-1} - 1$ roots.
Additionally, there may exist one root between $x = -\infty$ and the knot $x_1$ closest to $-\infty$, and another root between the knot $x_{m_{i-1}}$ closest to $\infty$ and $x = \infty$.
In total, $\w_{ik} \cdot \v_{i-1}(x) + b_{ik}$ can have at most $m_{i-1} + 1$ roots.
Hence, the output $\sigma(\w_{ik} \cdot \v_{i-1}(x) + b_{ik})$ of neuron $k$ in layer $i$ can create at most $m_{i-1} + 1$ knots, with the equality being met with a sawtooth wave.
Furthermore, each neuron $k$ can adjust $b_{ik}$ so as to create $m_{i-1} + 1$ knots uniquely.
This construction is demonstrated in \autoref{fig:sawtooth}(c).

Having shown that all knots can be preserved in every layer, and having computed the maximum number of knots that each neuron can create, the upper bound (\autoref{thm:bound}) can be formally derived.
Note that we have not yet shown that all knots can always be preserved at the same time that every neuron in every layer creates the maximum possible number of knots.
We first prove the upper bound as follows, and demonstrate the tightness of the bound by construction later in \autoref{sec:tight}.

\begin{proof}[Proof of \autoref{thm:bound}]
    For $l = 1$, the neural network can have up to one knot per neuron, as previously stated.
    That is, $m_1 \le n_1$, which is equivalent to~\eqref{eq:bound}.

    For $l > 1$, let us once again denote the number of knots in the affine transformation of layer $i$ outputs by $m_i$.
    In layer $i$, each neuron $j = 1, \dots, n_i$ can preserve at most all $m_{i-1}$ knots from the previous layer, and can also create at most $m_{i-1} + 1$ knots uniquely.
    Therefore, the upper bound on $m_i$ is
    \begin{subequations}
        \begin{align}
            \label{eq:induction-original}
            m_i &\le m_{i-1} + n_i (m_{i-1} + 1) \\
            \label{eq:induction}
            &= (n_i + 1) m_{i-1} + n_i.
        \end{align}
    \end{subequations}
    Setting $i = l + 1$ in~\eqref{eq:induction}, we have that $m_{l+1} \le (n_{l+1} + 1) m_l + n_{l+1}$.
    Supposing that \eqref{eq:bound} is true, we find that
    \begin{subequations}
        \begin{align}
            m_{l+1} &\le (n_{l+1} + 1) \sum_{i=1}^l n_i \prod_{j=i+1}^l (n_j + 1) + n_{l+1} \\
            &= \sum_{i=1}^l n_i \prod_{j=i+1}^{l+1} (n_j + 1) + n_{l+1} \\
            &= \sum_{i=1}^{l+1} n_i \prod_{j=i+1}^{l+1} (n_j + 1).
        \end{align}
    \end{subequations}
    Hence, if~\eqref{eq:bound} holds for $l$, then it also holds for $l + 1$, and the induction is complete.
\end{proof}

\begin{remark}
    The dimension $p$ of the output space does not affect the upper bound on the number of knots in the neural network; see Lemma~2 of \citet{Pascanu14}.
    The output layer is simply an affine transformation, and does not contain any rectified linear units.
    Therefore, all knots that are outputted from the final hidden layer $\v_l$ can be preserved.
    Additionally, some knots may possibly be destroyed in the degenerate case where $\v_l$ has discontinuities in its first derivative, but $\w_{l+1,k} \cdot \v_l$ does not for all $k = 1, \dots, p$.
    Either way, no new knots can be created in the output layer.
\end{remark}

\begin{remark}
    In most applications of neural networks,  $n_1 \ge \dots \ge n_l$, where $n_l$ is notably larger than unity.
    In this case, the upper bound~\eqref{eq:bound} is dominated by the $i = 1$ summand, and the upper bound is approximately
    \begin{equation}
        \prod_{i=1}^l n_i.
    \end{equation}
    If we further assume that
    \begin{equation}
        \label{eq:equal-n}
        n := n_1 = \dots = n_l
    \end{equation}
    (which is sometimes useful for analytical purposes but less commonly employed in practice), then the upper bound further reduces to $n^l$.
    This approximate upper bound is consistent with the tight asymptotic upper bound $\O(n^{q l})$ given by \citet{Raghu16}, where we have used the input dimension $q = 1$.
\end{remark}

\begin{remark}
    The number of scalar parameters in the weights and biases of a deep $\Reals^q \to \Reals^p$ network~\eqref{eq:deep} is
    \begin{equation}
        (q + 1) n_1
        + \sum_{i=1}^{l-1} (n_i + 1) n_{i+1}
        + (n_l + 1) p.
    \end{equation}
    If we assume~\eqref{eq:equal-n} once again, then for $q = 1$, the number of parameters is $2 n + (n + 1) (n (l - 1) + p) \approx (p + 2) n + (l - 1) n^2$.
    This number is typically far smaller than $n^l$ for $l \ge 3$.
    Thus, deep networks can possibly create a large number of knots with a comparatively small number of parameters.
    This feature plays a key role in the expressive power of deep neural networks.
    It has been suggested that although shallow networks can create models identical to deep networks via the universal approximation theorem, they may require many more parameters to do so; see \citet{Lin16} and the references within.
\end{remark}

\section{Tightness of the upper bound}
\label{sec:tight}

Next, we show that the upper bound~\eqref{eq:bound} is tight if there is a sufficient number of neurons in each layer, which will almost certainly be satisfied in practical applications.
This demonstration proceeds by construction.
In Lemma~\ref{lem:tight-1layer}, we first review the trivial case where the neural network has $l = 1$ layer.
We then show in Lemma~\ref{lem:sawtooth-layer2} that the affine transformation of the first hidden layer outputs can be made into a sawtooth wave.
Then, we show in Lemma~\ref{lem:sawtooth-inductive} that subsequent hidden layers can turn sawtooth wave inputs into sawtooth wave outputs with the maximum number of knots.
Finally, we reaffirm that all knots from a previous layer can be preserved in the application of a new layer, while creating the maximum number of knots.

\begin{lemma}
    \label{lem:tight-1layer}

    The upper bound~\eqref{eq:bound} is tight for single-hidden-layer neural networks.
\end{lemma}

\begin{proof}
    Equation~\eqref{eq:bound} reduces to $m_1 \le n_1$ for $l = 1$.
    As previously stated, the equality is obtained simply by choosing $b_{1j}$ and $w_{1j}$ in~\eqref{eq:shallow-nn} such that $x_j = -b_{1j} / w_{1j}$ is unique for each $j = 1, \dots, n_1$.
\end{proof}

\begin{lemma}
    \label{lem:sawtooth-layer2}

    If the first hidden layer has $n_1 \ge 3$ neurons, then there exist weights $w_{1j}, w_{2kj}$ and biases $b_{1j}, b_{2k}$ such that the input
    \begin{equation}
        \label{eq:layer2-input}
        \sum_{j=1}^{n_1} w_{2kj} \sigma(w_{1j} x + b_{1j}) + b_{2k}
    \end{equation}
    to the rectified linear unit in neuron $k$ of layer $2$ is a sawtooth wave.
\end{lemma}

\begin{proof}
    One way to construct such a sawtooth wave is to select
    \begin{subequations}
        \label{eq:sawtooth-parameters}
        \begin{align}
            w_{1j} &=
            \begin{cases}
                -1 &| \quad j = 3 \\
                1 &| \quad j \ne 3
            \end{cases} \\
            b_{1j} &=
            \begin{cases}
                j - 1 &| \quad j = 3 \\
                -j + 1 &| \quad j \ne 3
            \end{cases} \\
            w_{2kj} &=
            \begin{cases}
                \frac{3}{2} &| \quad j = 1 \\
                -1 &| \quad j \text{ even} \\
                1 &| \quad j > 1 \text{ and } j \text{ odd}
            \end{cases},
        \end{align}
    \end{subequations}
    with $b_{2k}$ arbitrary.
    This is more apparent if we apply Lemma~\ref{lem:forward-relu} and write~\eqref{eq:layer2-input} as
    \begin{equation}
        \sum_{j=1}^{n_1} s_{kj} \sigma(x - x_j) + c_{1k} x + c_{0k},
    \end{equation}
    where
    \begin{align}
        \label{eq:alt-parameters}
        x_j &= j - 1, &
        s_{kj} &= w_{2kj}, &
        c_{1k} &= -1, &
        c_{0k} &= b_{2k} + 2.
    \end{align}
    That is, the knots are evenly spaced, the initial slope from $x = -\infty$ to the first knot $x_1 = 0$ is $c_{1k} = -1$, and the slopes of the subsequent segments between knots are obtained by cumulatively adding $s_{kj}$.
    Thus, the slopes in successive linear pieces of the spline are
    \begin{equation}
        \label{eq:successive-slopes}
        \left\{c_{1k} + \sum_{j=1}^r s_{kj}\right\}_{r=0}^{n_1}
        = \left\{-1, \frac{1}{2}, -\frac{1}{2}, \frac{1}{2}, -\frac{1}{2}, \dots\right\},
    \end{equation}
    which generates a sawtooth wave.
    See \autoref{fig:sawtooth-layer2}
    \begin{figure}[t]
        \centering
        \includegraphics{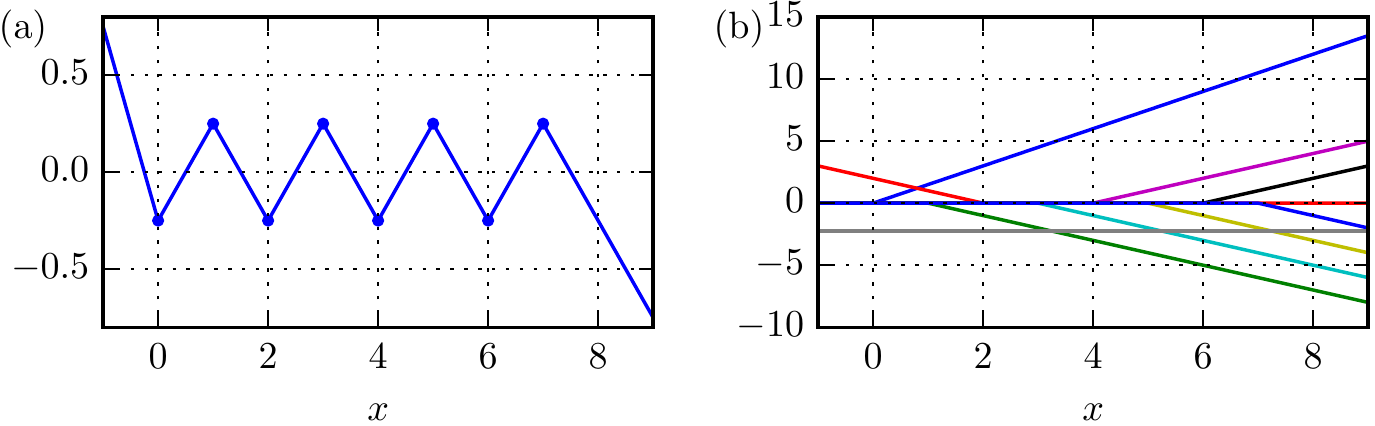}

        \caption{%
            (a)~The affine transformation~\eqref{eq:layer2-input} of first-hidden-layer outputs using the parameters in~\eqref{eq:sawtooth-parameters} with $n = 8$ and $b_{2k} = -9 / 4$.
            (b)~The rectified linear unit summands in (a), with each summand in a different non-gray color (see~\eqref{eq:layer2-input}), and the bias $b_{2k}$ in gray.%
        }
        \label{fig:sawtooth-layer2}
    \end{figure}
    for an example.
\end{proof}

\begin{lemma}
    \label{lem:sawtooth-inductive}

    Suppose layer $i \ge 2$ has $n_i \ge 3$ neurons, and there exist weights $\alpha_{ij} \in \Reals$ for $j = 1, \dots, n_{i-1}$ such that
    \begin{equation}
        \label{eq:sawtooth-in}
        g_i(x) := \sum_{j=1}^{n_{i-1}} \alpha_{ij} v_{i-1, j}(x)
    \end{equation}
    (which is an input to a layer $i$ rectified linear unit, up to a bias) is a sawtooth wave with $m_{i-1}$ knots.
    Then there exist weights $w_{ikj}, \alpha_{i+1, k} \in \Reals$ and biases $b_{ik} \in \Reals$ for $j = 1, \dots, n_{i-1}$ and $k = 1, \dots, n_i$ such that given
    \begin{equation}
        \label{eq:sawtooth-relu}
        v_{ik}(x) := \sigma\left(\sum_{j=1}^{n_{i-1}} w_{ikj} v_{i-1, j}(x) + b_{ik}\right),
    \end{equation}
    the function
    \begin{equation}
        \label{eq:sawtooth-out}
        g_{i+1}(x) := \sum_{k=1}^{n_i} \alpha_{i+1,k} v_{ik}(x)
    \end{equation}
    is a sawtooth wave with the maximal number of knots
    \begin{equation}
        \label{eq:max-knots}
        m_i = m_{i-1} + n_i (m_{i-1} + 1)
    \end{equation}
    (cf.~\eqref{eq:induction-original}).
\end{lemma}

\begin{proof}
    Suppose that---excluding the sections of $g_i(x)$ between $x = -\infty$ and the first knot $x_1$, and between the last knot $x_{m_{i-1}}$ and $x = \infty$---the minimum and maximum of the oscillation in $g_i(x)$ are respectively $g_\text{min}$ and $g_\text{max}$.
    For convenience, let us rescale $g_i(x)$ such that the minimum and maximum are respectively 0 and 1; we define
    \begin{equation}
        \label{eq:transformation}
        \hat{g}_i(x) := \frac{g_i(x) - g_\text{min}}{g_\text{max} - g_\text{min}}.
    \end{equation}
    The central idea behind the construction is to select the weights and biases so that every line segment of the oscillation between $\hat{g}_i = 0$ and 1 is transformed into a sawtooth wave with $n_i$ knots.

    One method to achieve this is to construct the wave
    \begin{equation}
        \begin{split}
            \label{eq:sawtooth-induction}
            g_{i+1}(x)
            &= \frac{3}{2} \sigma\left(\hat{g}_i(x) - \frac{1}{2 n_i + 1}\right)
            - \sigma\left(\hat{g}_i(x) - \frac{3}{2 n_i + 1}\right) \\
            & \hspace{1.2em} + \sigma\left(-\hat{g}_i(x) + \frac{5}{2 n_i + 1}\right)
            + \sum_{k=4}^{n_i} (-1)^{k+1} \sigma\left(\hat{g}_i(x) - \frac{2 k - 1}{2 n_i + 1}\right).
        \end{split}
    \end{equation}
    This construction has a natural equivalence
    with~\eqref{eq:sawtooth-parameters}, with $\hat{g}_i$ used in
    place of~$x$.  Interpreting $\hat{g}_i$ as the independent
    variable and setting
    \begin{align}
        \label{eq:alpha}
        \alpha_{i+1,k} &:=
        \begin{cases}
            \frac{3}{2} &| \quad k = 1 \\
            -1 &| \quad k \text{ even} \\
            1 &| \quad k > 1 \text{ and } k \text{ odd}
        \end{cases}, &
        \gamma_k &:= \frac{2 k - 1}{2 n_i + 1},
    \end{align}
    we employ~\eqref{eq:relu-negative} to find that \eqref{eq:sawtooth-induction} is equivalent to
    \begin{equation}
        \label{eq:induction-equivalent}
        g_{i+1} = \sum_{k=1}^{n_i} \alpha_{i+1,k} \sigma(\hat{g}_i - \gamma_k)
        - \hat{g}_i + \frac{5}{2 n_i + 1}.
    \end{equation}
    Thus, as $\hat{g}_i$ increases from 0 to 1, the slope of $g_{1+i}$ with respect to $\hat{g}_i$ in consecutive segments is
    \begin{equation}
        \label{eq:inductive-slopes}
        \left\{-1 + \sum_{k=1}^r \alpha_{i+1,k}\right\}_{r=0}^{n_i}
        = \left\{-1, \frac{1}{2}, -\frac{1}{2}, \frac{1}{2}, -\frac{1}{2}, \dots\right\},
    \end{equation}
    (cf.~\eqref{eq:successive-slopes} and see~\autoref{fig:sawtooth-induction}).
    \begin{figure}[t]
        \centering
        \includegraphics{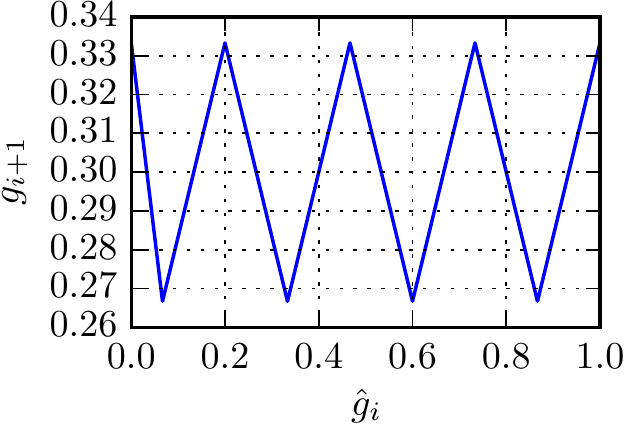}

        \caption{The wave~\eqref{eq:sawtooth-induction} with $n_i = 7$.}
        \label{fig:sawtooth-induction}
    \end{figure}
    Hence, for every line segment of $\hat{g}_i(x)$ between consecutive knots, $g_{i+1}(x)$ is a sawtooth wave with $n_i$ knots.

    Referring back to \autoref{sec:upper-bound}, we recall that the maximum number of knots~\eqref{eq:max-knots} is achieved if every knot in $\hat{g}_i(x)$ is retained, and each of the $n_i$ neurons uniquely creates $m_{i-1} + 1$ knots.
    We verify that these conditions are met.
    The quantity $\hat{g}_i - \gamma_k$ has a total of $m_{i-1} - 1$ roots between the $m_{i-1}$ knots, plus one each between $x = -\infty$ and the first knot $x_1$, and between the last knot $x_{m_{i-1}}$ and $x = \infty$.
    In total, each neuron creates $m_{i-1} + 1$ knots.
    Furthermore, each bias $\gamma_k$ is unique, ensuring that the knots that are created by each of the $n_i$ rectified linear units are also unique (see \autoref{fig:sawtooth}(c)).
    Finally, since the operand to $\sigma$ in the third summand in~\eqref{eq:sawtooth-induction} contains $-\hat{g}_i(x)$ as opposed to $\hat{g}_i(x)$ in all other summands, both the lower and the upper knots of the sawtooth wave are preserved by the right-hand side of~\eqref{eq:sawtooth-induction}, as shown in \autoref{fig:sawtooth}(b).

    Note that for the induction to carry through successive layers, we must also verify that the local minima of~\eqref{eq:induction-equivalent} are all equal, as are the local maxima.
    This is easily confirmed, since the spacing in $\hat{g}_i$ between consecutive knots (including endpoints) is
    \begin{equation}
        \{\gamma_1 - 0, \gamma_2 - \gamma_1, \dots, \gamma_{n_i} - \gamma_{n_i - 1}, 1 - \gamma_{n_i}\}
        = \left\{\frac{1}{2 n_i + 1}, \frac{2}{2 n_i + 1}, \dots, \frac{2}{2 n_i + 1}\right\}.
    \end{equation}
    Comparing this against the slopes~\eqref{eq:inductive-slopes}, the vertical displacement between consecutive knots is simply
    \begin{equation}
        \left\{-\frac{1}{2 n_i + 1}, \frac{1}{2 n_i + 1}, -\frac{1}{2 n_i + 1}, \frac{1}{2 n_i + 1}, \dots\right\}.
    \end{equation}

    Finally, to complete the construction, we combine (\ref{eq:sawtooth-in},~\ref{eq:sawtooth-relu},~\ref{eq:transformation},~\ref{eq:sawtooth-induction}) to find that one valid set of weights and biases is given by~\eqref{eq:alpha} and
    \begin{subequations}
        \label{eq:induction-params}
        \begin{align}
            \label{eq:induction-weight}
            w_{ikj} &= \frac{\alpha_{ij}}{g_\text{max} - g_\text{min}} \cdot
            \begin{cases}
                -1 &| \quad k = 3 \\
                1 &| \quad k \ne 3
            \end{cases} \\
            \label{eq:induction-bias}
            b_{ik} &= -\left(\frac{g_\text{min}}{g_\text{max} - g_\text{min}} + \frac{2 k - 1}{2 n_i + 1}\right)
            \cdot
            \begin{cases}
                -1 &| \quad k = 3 \\
                1 &| \quad k \ne 3
            \end{cases}.
        \end{align}
    \end{subequations}
\end{proof}

With these lemmas in place, the tightness of the upper bound (\autoref{thm:tight}) can now be proven.

\begin{proof}[Proof of \autoref{thm:tight}]
    For $i = 1, \dots, l - 1$, the inductive and constructive proof is given quite simply by the combination of Lemmas~\ref{lem:tight-1layer}--\ref{lem:sawtooth-inductive}.
    In the base case, Lemma~\ref{lem:tight-1layer} shows that~\eqref{eq:bound} is tight for $l = 1$.
    Next, Lemma~\ref{lem:sawtooth-layer2} shows that the affine transformation of the first hidden layer outputs---whether it is for the output of a single-hidden-layer neural network, or for a second hidden layer in a deep network---can be made into a sawtooth wave.
    In light of Lemma~\ref{lem:tight-1layer}, this sawtooth wave can be constructed with the maximal $m_1 = n_1$ knots.

    Next, the induction step is given by Lemma~\ref{lem:sawtooth-inductive}.
    Namely, suppose that the affine transformation of the layer $i - 1$ outputs is a sawtooth wave with the maximal number of knots $m_{i-1}$.
    Then, it is possible to construct a sawtooth wave out of an affine transformation of the layer $i$ outputs, such that the wave also has the maximal number of knots $m_i = m_{i-1} + n_i (m_{i-1} + 1)$.
    This induction step can be carried out sequentially from the second hidden layer $i = 2$ all the way to the penultimate hidden layer $i = l - 1$.

    Finally, we note that the final hidden layer $i = l$ deserves special treatment because the output layer does not contain any rectified linear units.
    As a direct result, it is not actually necessary for the final hidden layer to output a sawtooth wave.
    \autoref{sec:example} will later demonstrate this idea in an example.
    Instead, it is sufficient to have two neurons in the final hidden layer and still maintain the induction relation~\eqref{eq:induction-original}.
    By referring back to \autoref{fig:sawtooth}(b), we remind that two neurons can preserve all $m_{i-1}$ knots from the penultimate layer, while each uniquely introducing $m_{i-1} + 1$ new knots with the application of the rectified linear unit.
\end{proof}

In the constructive proofs of Lemmas~\ref{lem:sawtooth-layer2} and~\ref{lem:sawtooth-inductive}, it is apparent that special consideration has been given to the third neuron in the respective series.
This is also evident in~\autoref{fig:sawtooth-layer2}(b), which shows that the sawtooth wave in the affine transformation of the first hidden layer outputs can be constructed from all forward-facing rectified linear units, except for the third unit which faces backwards.
To construct a sawtooth wave, it is in fact necessary to reverse the orientation of neuron $j$ for some $j \ge 3$.
Since a maximally high-wavenumber wave must be input into every rectified linear unit to meet the upper bound, an additional result is the following corollary, which is essentially the inverse of \autoref{thm:tight}.
We remark that the conditions of this corollary may not be seen in practice, but we nevertheless state this result for completeness.

\begin{corollary}
    For deep neural networks with $l \ge 2$ layers, the upper bound in~\eqref{eq:bound} is not tight if $n_i < 3$ for any $i = 1, \dots, l - 1$, or if $n_l = 1$.
\end{corollary}

\begin{proof}
    For the upper bound to be met with $l \ge 2$, the affine transformations of the outputs of hidden layers $i = 1, \dots, l - 1$ must have alternating slopes---i.e., between positive and negative---through all linear pieces.
    Only then can each rectified linear unit in layer $i + 1$ create the maximal $m_i + 1$ unique knots.
    This condition can be analyzed separately for $i = 1$ and $i > 1$.

    For $i = 1$, the individual rectified linear units of the first hidden layer must be linearly combined to construct a sawtooth wave; see Lemma~\ref{lem:sawtooth-layer2}.
    Such an arrangement is not possible in the (rather unorthodox) case of $n_1 = 1$ or 2.
    The case where $n_1 = 1$ is trivial: the function $\sigma(w_{11} x + b_{11})$ clearly cannot have both a negative and a positive slope for a given choice of $w_{11}$ and $b_{11}$.
    The case where $n_1 = 2$ is slightly less obvious.
    Suppose, without loss of generality, that we wish to construct a linear combination
    \begin{equation}
        g_2(x) = \sum_{j=1}^2 w_{2j} \sigma(w_{1j} x + b_{1j})
    \end{equation}
    of two neural outputs in the first hidden layer that slopes down, then up, and finally down again: \tikz \draw (0, 0.3) -- (0.2, 0) -- (0.4, 0.3) -- (0.6, 0);.
    The left and right extremes of this shape requires that one neuron be oriented toward quadrant II (\tikz \draw (0, 0.3) -- (0.2, 0) -- (0.6, 0);) and the second neuron be oriented toward quadrant IV (\tikz \draw (0, 0.3) -- (0.4, 0.3) -- (0.6, 0);).
    There does not exist a way to sum these two rectified linear units and obtain the positive slope in the middle segment of the linear combination.
    Therefore, the upper bound~\eqref{eq:bound} cannot be achieved if $n_1 < 3$.

    For $i = 2, \dots, l - 1$, hidden layer $i$ must be able to transform a sawtooth wave with $m_{i-1}$ knots into another sawtooth wave with $m_{i-1} + n_i (m_{i-1} + 1)$ knots.
    Consider a single line segment in the linear combination of layer $i - 1$ outputs.
    Using the notation of Lemma~\ref{lem:sawtooth-inductive}, if the output of this segment has a minimum $g_i = g_\text{min}$ and maximum $g_i = g_\text{max}$, then we require some choice of $w_{ij}$, $w_{i+1,j}$ and $b_{ij}$ such that the derivative of
    \begin{equation}
        g_{i+1} = \sum_{j=1}^{n_i} w_{i+1,j} \sigma(w_{ij} g_i + b_{ij})
    \end{equation}
    with respect to $g_i$ contains $n_i$ sign changes as $g_i$ increases from $g_\text{min}$ to the next instance of $g_\text{max}$.
    Using the same argument as the previous paragraph for $i = 1$, but using the input $g_i$ in place of $x$, such an arrangement is impossible if $n_i = 1$ or 2.

    Finally, we make the observation that in the unusual case that $n_l = 1$, it is impossible for that single final-hidden-layer neuron both to preserve all $m_{l-1}$ knots from the penultimate layer, while also introducing $m_{l-1} + 1$ knots.
    If $m_{l-1} + 1$ knots were introduced by drawing a bias through the sawtooth wave from layer $l - 1$, then half of the $m_{l-1}$ knots (rounded up or down, if $m_{l-1}$ is odd) from the previous layer would be discarded.
    Alternatively, if the single neuron preserved all $m_{l-1}$ knots from the previous layer, then it would not be able to create new knots, as required by the upper bound.
\end{proof}

\section{Example construction of tight upper bound}
\label{sec:example}

In this section, we demonstrate a construction of an $\Reals \to \Reals^p$ neural network with a number of knots exactly equal to the upper bound.
For the sake of keeping the neural network size manageable, we intentionally use a small number of neurons.
We choose to have $l = 3$ hidden layers, with $n_1 = 6$ neurons in the first layer, $n_2 = 3$ neurons in the second layer, and $n_3 = 2$ neurons in the third layer.
We will employ $p = 2$ in this example, though as \autoref{sec:upper-bound} shows, the output dimension is actually irrelevant to the number of knots in the neural network.

Using these values in~\eqref{eq:bound}, we find that the upper bound on the number of knots is $m_1 = 6$ in the first layer outputs, $m_2 = 27$ in the second layer outputs, and $m_3 = 83$ in the third layer and final outputs.
Since $n_1$, $n_2$, and $n_3$ satisfy the criteria in \autoref{thm:tight}, these bounds are tight, and we can use the constructions in \autoref{sec:tight} to define a neural network with these numbers of knots.

The example neural network is given by the equations
\begin{subequations}
    \label{eq:example-outputs}
    \begin{align}
        v_{1k} &= \sigma(w_{1k} x + b_{1k}),
        & k = 1, \dots, n_1 \\
        v_{2k} &= \sigma\left(\sum_{j=1}^{n_1} w_{2kj} v_{1j} + b_{2k}\right),
        & k = 1, \dots, n_2 \\
        v_{3k} &= \sigma\left(\sum_{j=1}^{n_2} w_{3kj} v_{2j} + b_{3k}\right),
        & k = 1, \dots, n_3 \\
        \label{eq:example-output}
        y_k &= \sum_{j=1}^{n_3} w_{4kj} v_{3j} + b_{4k},
        & k = 1, \dots, p,
    \end{align}
\end{subequations}
where
\begin{subequations}
    \label{eq:params-1}
    \begin{align}
        w_{1k} &=
        \begin{cases}
            -1 &| \quad k = 3 \\
            1 &| \quad k \ne 3
        \end{cases}, \\
        b_{1k} &=
        \begin{cases}
            k - 1 &| \quad k = 3 \\
            -k + 1 &| \quad k \ne 3
        \end{cases}
    \end{align}
\end{subequations}
for $k = 1, \dots, n_1$,
\begin{subequations}
    \label{eq:params-2}
    \begin{align}
        \label{eq:weights-2}
        w_{2kj} &= 2 w_{1k} \cdot
        \begin{cases}
            \frac{3}{2} &| \quad j = 1 \\
            -1 &| \quad j \text{ even} \\
            1 &| \quad j > 1 \text{ and } j \text{ odd}
        \end{cases}, \\
        \label{eq:biases-2}
        b_{2k} &= \left(-4 - \frac{2 k - 1}{2 n_2 + 1}\right) w_{1k}
    \end{align}
\end{subequations}
for $k = 1, \dots, n_2$,
\begin{subequations}
    \label{eq:params-3}
    \begin{align}
        \label{eq:weights-3}
        w_{3kj} &= 7 (-1)^{k-1} \cdot
        \begin{cases}
            \frac{3}{2} &| \quad j = 1 \\
            -1 &| \quad j \text{ even} \\
            1 &| \quad j > 1 \text{ and } j \text{ odd}
        \end{cases}, \\
        \label{eq:biases-3}
        b_{3k} &= (-1)^{k-1} \left(-4 - \frac{k}{n_3 + 1}\right)
    \end{align}
\end{subequations}
for $k = 1, \dots, n_3$, and
\begin{subequations}
    \label{eq:params-4}
    \begin{align}
        w_{4kj} &= (-1)^{j+k}, \\
        b_{4k} &= k - 1
    \end{align}
\end{subequations}
for $k = 1, \dots, p$.

The hidden layer and model outputs for this example are shown in \autoref{fig:example}.
\begin{figure}[!t]
    \centering
    \includegraphics{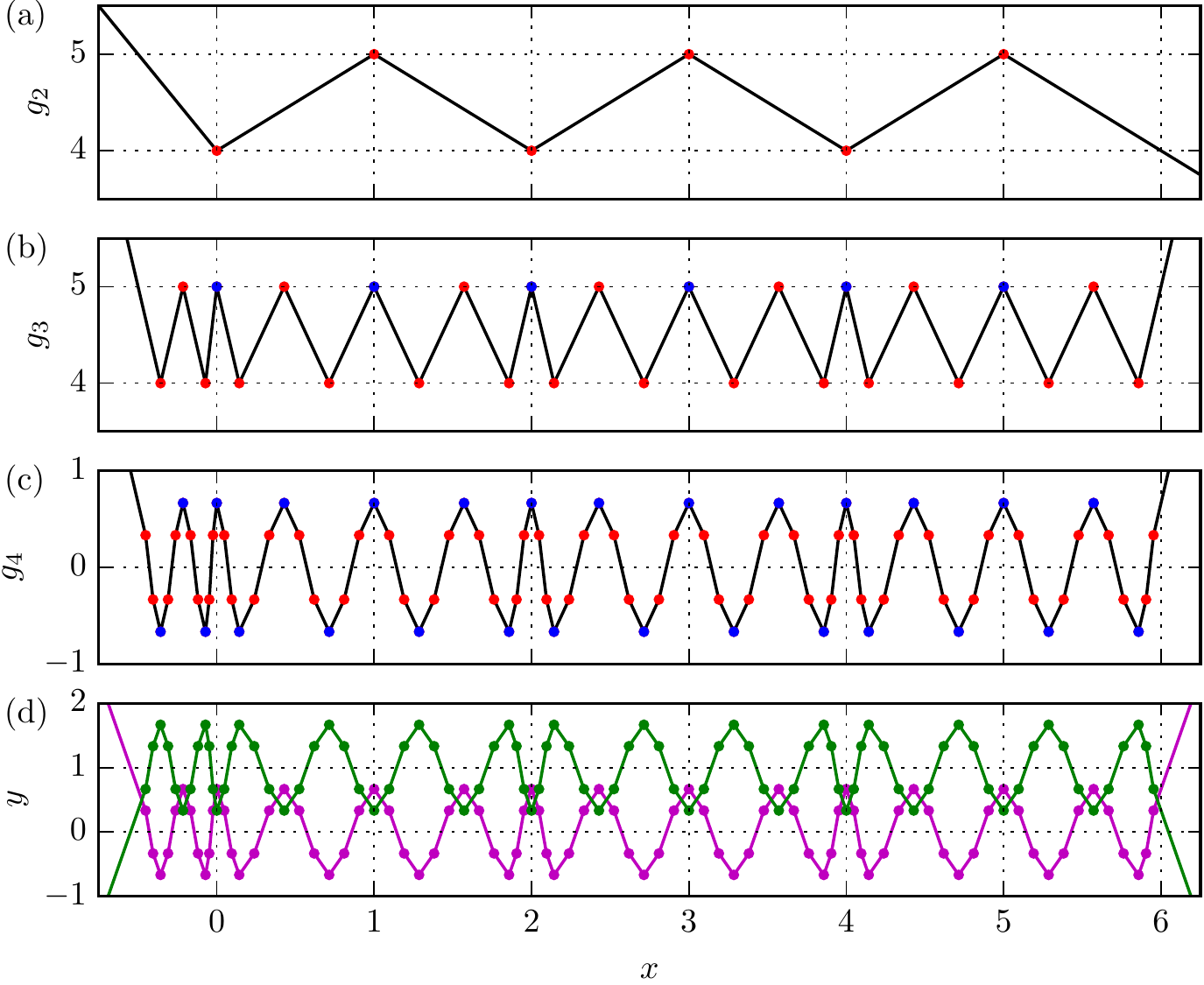}
    \caption{%
        The neural network given by~(\ref{eq:example-outputs}--\ref{eq:params-4}), as an example of a model that meets the upper bound~\eqref{eq:bound} on the number of knots.
        The sawtooth waves $\sum_{j=1}^{n_i} w_{i+1,1,j} v_{ij}$ are constructed by linearly combining the outputs of hidden layer (a)~$i = 1$, with six knots; (b)~$i = 2$, with 27 knots; and (c)~$i = 3$, with 83 knots.
        Knots retained from the previous layer are shown in blue, and knots created in the current layer are shown in red.
        (d)~The outputs $y_1$ (magenta) and $y_2$ (green), with 83 knots.%
    }
    \label{fig:example}
\end{figure}
The interpretation of the above weights and biases proceeds as follows.
In the first hidden layer, $w_{1k}$ and $b_{1k}$~\eqref{eq:params-1}, as well as the dependence of $w_{2kj}$~\eqref{eq:weights-2} on $j$, are copied directly from the construction for a sawtooth wave~\eqref{eq:sawtooth-parameters} in Lemma~\ref{lem:sawtooth-layer2}.
Thus, they create knots at $x = 0, \dots, n_1 - 1$, and the rectified linear units are oriented as in \autoref{fig:sawtooth-layer2}(b).
The factor of 2 in~\eqref{eq:weights-2} is added for convenience to make the sawtooth span a range of 1 instead of $1 / 2$.
The sawtooth wave
\begin{equation}
    g_2(x) = \sum_{j=1}^{n_1} w_{21j} v_{1j}(x)
\end{equation}
that is used in neuron $k = 1$ of layer $i = 2$ is shown in \autoref{fig:example}(a).

From this figure, we observe that the range of the sawtooth wave, excluding the end parts with $g_2 \to \pm \infty$, is $[4, 5]$.
Following~\eqref{eq:induction-params}, we flip the signs of $w_{2kj}$ and $b_{2k}$~\eqref{eq:params-2} for $k = 3$.
Furthermore, we set $b_{2k}$ according to~\eqref{eq:induction-bias}, so that each neuron offsets $g_2$ by the proper amount to construct $m_1 + 1$ unique knots, which can then be rearranged into a new sawtooth wave.
In addition, we set the dependence of $w_{3kj}$~\eqref{eq:weights-3} on $j$ to match the construction in~\eqref{eq:sawtooth-induction}.
As shown in \autoref{fig:example}(b), this choice of parameters produces the sawtooth wave
\begin{equation}
    g_3(x) = \sum_{j=1}^{n_2} w_{31j} v_{2j}(x)
\end{equation}
that is used in neuron $k = 1$ of layer $i = 3$.
We may observe from this figure that this second layer output retains all the knots from the first layer output (\autoref{fig:example}(a)), and it also creates the maximal $n_2$ knots between all the knots of the first layer output, as well as in $(-\infty, 0)$ and $(n_1 - 1, \infty)$.

Moving forward, the construction of the third hidden layer in this example proceeds differently.
As stated in \autoref{thm:tight}, the final hidden layer $i = 3$ only needs to have $n_3 = 2$ neurons to meet the tight upper bound, since there are no further rectified linear units and the sawtooth waveform is therefore no longer required.
By following the strategy shown in \autoref{fig:sawtooth}(c), we pick $w_{3kj}$ and $b_{3k}$~\eqref{eq:params-3} to have opposite signs between $k = 1$ and 2.
Furthermore, we note that the sawtooth in \autoref{fig:example}(b) has a range of $[4, 5]$, so we pick $b_{3k}$ to be two different values for $k = 1$ and 2 within the range $(-5, -4)$.
That way, as shown in \autoref{fig:sawtooth}(b), the $k = 1$ neuron retains the upper knots of \autoref{fig:example}(b), while the $k = 2$ neuron retains the lower ones.
Furthermore, each of the two neurons produces one new knot in the $m_2 + 1$ regions of $\Reals$ divided by the knots of \autoref{fig:example}(b).
The factor of seven in~\eqref{eq:weights-3} is arbitrary.

Finally, the choice of the output weights $w_{4kj}$ and biases $b_{4k}$~\eqref{eq:params-4} is also arbitrary, since the output layer does not contain rectified linear units and cannot destroy or create knots.
The sawtooth wave
\begin{equation}
    g_4(x) = \sum_{j=1}^{n_3} w_{41j} v_{3j}(x)
\end{equation}
that makes up the output $y_1$ is shown in \autoref{fig:example}(c).
The neural network outputs~\eqref{eq:example-output}, with the maximal $m_3 = 83$ knots, are shown in \autoref{fig:example}(d).

\section{Conclusion}
\label{sec:conclusion}

We have shown that deep, fully-connected, $\Reals \to \Reals^p$ neural networks with rectified linear unit activations are essentially linear splines.
In \autoref{thm:bound}, we derived an upper bound on the number of knots that such neural networks can have.
The upper bound is given exactly by~\eqref{eq:bound}; to close approximation, this bound is $n_1 \cdots n_l$.
We then showed in \autoref{thm:tight} that the upper bound is tight for the neural network widths that would be encountered in practice.
An example of a deep neural network exactly meeting this upper bound was described in \autoref{sec:example}.

It is clear from the setup of the upper bound that the imposed conditions are prohibitively restrictive.
Most notably, it is common in practical applications to construct $\Reals^q \to \Reals^p$ neural networks where $q$ may be on the order of $10^3$ or even larger.
As aforementioned, previous works have computed approximate or asymptotic bounds on the number of linear pieces in $\Reals^q \to \Reals^p$ neural networks \citep{MontufarNIPS14,Pascanu14,Raghu16}.
Nevertheless, an exact upper bound---let alone a tight one---remains to be derived in this generic case.

In addition, there is little reason to believe that neural networks used in actual applications would contain a number of knots equal to or close to the upper bound presented here.
The construction of the upper bound required that a sawtooth wave be constructed at every hidden layer except the final one.
It is unlikely that such maximally high-wavenumber networks would be fitted to actual data, and the likelihood is even lower for large input dimensions $q$ commonly used in practice.

Thus, the results of this paper can be interpreted as a theoretical ``brick-wall'' limit on neural network expressivity, which may be used as a guideline or check in designing actual neural networks.
Two companion papers present more realistic scenarios.
In the first \citep{ChenNIPS16}, we explore the number of knots in randomly weighted and biased neural networks.
In the second \citep{WalkerSCAMP16}, we describe empirical results on the behavior of neural network training.
Both of these scenarios are more representative of actual situations seen in practice.
Not only is a random neural network more likely to represent an ``average case'' neural network rather than a ``best case,'' but also---as demonstrated in \citet{ChenNIPS16}---random neural networks are actually encountered in the early stages of training on data.
In \citet{ChenNIPS16}, we also describe open problems related to the expressivity of neural networks in greater detail.
These papers are still largely analytical in nature, since the chief objective of our investigation is to close the gap between our understanding of neural network theory and applications.

\vspace{1em}
Alden Walker is gratefully acknowledged for providing \autoref{fig:2d-model} and for helpful conversations.
Discussions with Anthony Gamst were also very fruitful, and led to the central ideas of the work presented in \citet{ChenNIPS16,WalkerSCAMP16}.

\bibliography{knots_upper_bound}

\end{document}